\documentclass{article}
\pdfoutput=1
\usepackage[final]{nips_anon}

\usepackage[utf8]{inputenc} 
\usepackage[T1]{fontenc}    
\usepackage{hyperref}       
\usepackage{url}            
\usepackage{booktabs}       
\usepackage{amsfonts}       
\usepackage{nicefrac}       
\usepackage{microtype}      

\usepackage{graphicx} 
\usepackage{caption}
\usepackage{subcaption}
\usepackage{tikz}
\usepackage{dsfont}
\usepackage{amsmath}
\usepackage{amssymb}
\usepackage{amsthm}
\usepackage{times}
\usepackage{natbib}
\usepackage{algorithm}
\usepackage[noend]{algorithmic}
\usepackage{wrapfig}

\newif\ifsup\suptrue

\graphicspath{ {figures/} }
\usetikzlibrary{arrows,positioning,backgrounds} 
\tikzset{
    >=stealth',
    observed/.style={
           circle,
           rounded corners,
           draw=black, thick,
           minimum width=2.2em,
           minimum height=2.2em,
           font=\tiny,
           text centered,
           scale=.8,
           fill=blue!20!white},
     latent/.style={
           circle,
           rounded corners,
           draw=black, thick, dashed,
           minimum width=.5em,
           minimum height=.5em,
           font=\footnotesize,
           text centered,
           fill=black!10!white
           },
     empty/.style={
           circle,
           rounded corners,
           minimum width=.5em,
           minimum height=.5em,
           font=\footnotesize,
           text centered,
           },
    pil/.style={
           o->,
           thick,
           shorten <=2pt,
           shorten >=2pt,},
    sh/.style={ shade, shading=axis, left color=red, right color=green,
    shading angle=45 }  
}

\newcommand{\E}[1]{\mathbb E\left[#1\right]}
\newcommand{\R}{\mathbb R}
\newcommand{\Var}{\operatorname{Var}}

\newcommand{\set}[1]{\left\{#1\right\}}
\newcommand{\ind}[1]{\mathds{1}\!\!\set{#1}}
\newcommand{\argmax}{\operatornamewithlimits{arg\,max}}
\newcommand{\argmin}{\operatornamewithlimits{arg\,min}}

\newcommand{\eqn}[1]{\begin{align}#1\end{align}}
\newcommand{\eq}[1]{\begin{align*}#1\end{align*}}

\renewcommand{\P}[1]{\operatorname{P}\left\{#1\right\}}
\newcommand{\Pri}[1]{\operatorname{P}_i\left\{#1\right\}}
\newcommand{\Prz}[1]{\operatorname{P}_0\left\{#1\right\}}
\newcommand{\bigo}[1]{\mathcal{O}\left( #1 \right)}

\newcommand{\KL}{\operatorname{KL}}
\newcommand{\simpleregret}{R_T}

\newcommand{\Q}[1]{\operatorname{Q}\left\{#1\right\}}
\newcommand{\EE}{\mathbb E}
\newcommand{\EEa}{\EE_a}
\newcommand{\Pns}[2]{\operatorname{P}_{#1}\left\{#2\right\}}
\newcommand{\Pn}[2]{\operatorname{P}\left\{#2|#1\right\}}
\newcommand{\parents}[1]{\operatorname{\mathcal{P}a}_{#1}}
\newcommand{\actions}{\mathcal{A}}
\newcommand{\calA}{\mathcal A}

\newcommand{\ie}{\textit{i.e.}}
\newcommand{\eg}{\textit{e.g.}}

\newcommand{\Ps}{\operatorname{P}}

\theoremstyle{plain}
\newtheorem{theorem}{Theorem}
\newtheorem{proposition}[theorem]{Proposition}
\newtheorem{lemma}[theorem]{Lemma}

\theoremstyle{definition}
\newtheorem{definition}[theorem]{Definition}

\newtheorem{remark}[theorem]{Remark}

\let\epsilon\varepsilon

\title{Causal Bandits: Learning Good Interventions via Causal Inference}

\author{
  Finnian Lattimore \\
  Australian National University and Data61/NICTA \\
  \texttt{finn.lattimore@gmail.com} \\
   \And
   Tor Lattimore \\
   University of Alberta \\
   \texttt{tor.lattimore@gmail.com} \\
   \And
   Mark D. Reid \\
   Australian National University and Data61/NICTA \\
   \texttt{mark.reid@anu.edu.au} \\
}

\begin{document}

\maketitle

\begin{abstract} 
We study the problem of using causal models to improve the rate at which good interventions can be learned online in a stochastic environment. 
Our formalism combines multi-arm bandits and causal inference to model a novel type of bandit feedback that is not exploited by existing approaches.
We propose a new algorithm that exploits the causal feedback and prove a bound on its simple regret that is strictly better (in all quantities) 
than algorithms that do not use the additional causal information.
\end{abstract} 


\section{Introduction}
\label{sec:intro}
Medical drug testing, policy setting, and other scientific processes are commonly framed and analysed in the language of sequential experimental design and, in special cases, as bandit problems~\citep{Robbins1952,Chernoff1959}. 
In this framework, single actions (also referred to as interventions) from a pre-determined set are repeatedly performed in 
order to evaluate their effectiveness via feedback from a single, real-valued reward signal.
We propose a generalisation of the standard model by assuming that, in addition to the reward signal, the learner observes the values of a number of covariates 
drawn from a probabilistic causal model~\citep{Pearl2000}.
Causal models are commonly used in disciplines where explicit experimentation may be difficult such as social science, demography and economics.
For example, when predicting the effect of changes to childcare subsidies on workforce participation, or school choice on grades. 
Results from causal inference relate observational distributions to interventional ones, allowing the outcome of an intervention to be predicted without
explicitly performing it.
By exploiting the causal information we show, theoretically and empirically, how non-interventional observations can be used to improve the rate at 
which high-reward actions can be identified.

The type of problem we are concerned with is best illustrated with an example. 
Consider a farmer wishing to optimise the yield of her crop. 
She knows that crop yield is only affected by temperature, a particular soil nutrient, and moisture level but the precise effect of their combination is unknown.
In each season the farmer has enough time and money to intervene and control at most one of these variables:
deploying shade or heat lamps will set the temperature to be low or high; the nutrient can be added or removed through a choice of fertilizer; and irrigation or rain-proof covers will keep the soil wet or dry.
When not intervened upon, the temperature, soil, and moisture vary naturally from season to season due to weather conditions and these are all observed along with the final crop yield at the end of each season.
How might the farmer best experiment to identify the single, highest yielding intervention in a limited number of seasons?
\paragraph{Contributions} We take the first step towards formalising and solving problems such as the one above. 
In \S\ref{sec:defs} we formally introduce \emph{causal bandit problems} in which interventions are treated as arms in a bandit problem but their influence on the reward --- along with any other observations --- is assumed to conform to a known causal graph. 
We show that our causal bandit framework subsumes the classical bandits (no additional observations) and contextual stochastic bandit problems (observations are revealed before an intervention is chosen) before focusing on the case where, like the above example, observations occur \emph{after} each intervention is made.

Our focus is on the simple regret, which measures the difference between the return of the optimal action and that of the action chosen by the algorithm after $T$ rounds.
In \S\ref{sec:simple-regret} we analyse a specific family of causal bandit problems that we call \emph{parallel bandit} problems in which $N$ factors affect the reward independently and there are $2N$ possible interventions.
We propose a simple causal best arm identification algorithm for this problem and show that up to logarithmic factors it enjoys minimax optimal
simple regret guarantees of $\smash{\tilde\Theta(\sqrt{m/T})}$ where $m$ depends on the causal model and may be much smaller than $N$.
In contrast, existing best arm identification algorithms suffer $\smash{\Omega(\sqrt{N/T})}$ simple regret (Thm. 4 by \citet{audibert2010best}).
This shows theoretically the value of our framework over the traditional bandit problem. 
Experiments in \S\ref{sec:experiments} further demonstrate the value of causal models in this framework.

In the general casual bandit problem interventions and observations may have a complex relationship. 
In \S\ref{sec:simple-regret-general} we propose a new algorithm inspired by importance-sampling that a) enjoys sub-linear regret equivalent 
to the optimal rate in the parallel bandit setting and b) captures many of the intricacies of sharing information in a causal graph in the general case.
As in the parallel bandit case, the regret guarantee scales like $\smash{O(\sqrt{m/T})}$ where $m$ depends on the underlying causal structure, with 
smaller values corresponding to structures that are easier to learn. The value of $m$ is always less than the number of interventions $N$ and in the special
case of the parallel bandit (where we have lower bounds) the notions are equivalent.

\paragraph{Related Work} As alluded to above, causal bandit problems can be treated as classical multi-armed bandit problems by simply ignoring the causal model and extra observations and applying an existing best-arm identification algorithm with well understood simple regret guarantees \citep{Jamieson2013}. However, as we show in \S\ref{sec:simple-regret}, ignoring the extra information available in the non-intervened variables yields sub-optimal performance.

A well-studied class of bandit problems with side information are ``contextual bandits''~\cite{Langford2008,Agarwal2014}. Our framework bears a superficial similarity to contextual bandit problems since the extra observations on non-intervened variables might be viewed as context for selecting an intervention. 
However, a crucial difference is that in our model the extra observations are only revealed \emph{after} selecting an intervention and hence cannot be used as context. 

There have been several proposals for bandit problems where extra feedback is received after an action is taken.
Most recently, \citet{Alon2015}, \citet{Kocak2014} have considered very general models related to partial monitoring games~\citep{Bartok2014} where rewards on unplayed actions are revealed according to a feedback graph. As we discuss in \S\ref{sec:discussion}, the parallel bandit problem can be captured in this framework, however the regret bounds are not optimal in our setting. They also focus on cumulative regret, which cannot be used to guarantee low simple regret~\citep{Bubeck2009a}. The partial monitoring approach taken by \cite{wu2015online} could be applied (up to modifications for the simple regret) to the parallel bandit, but the resulting strategy would need to know the likelihood of each factor in advance, while our strategy learns this online. \citet{Yu2009} utilize extra observations to detect changes in the reward distribution, whereas we assume fixed reward distributions and use extra observations to improve arm selection. \citet{Avner2012} analyse bandit problems where the choice of arm to pull and arm to receive feedback on are decoupled. The main difference from our present work is our focus on simple regret and the more complex information linking rewards for different arms via causal graphs. To the best of our knowledge, our paper is the first to analyse simple regret in bandit problems with extra post-action feedback.


Two pieces of recent work also consider applying ideas from causal inference to bandit problems.
\citet{Bareinboim2015} demonstrate that in the presence of confounding variables the value that a variable would have taken had it not been 
intervened on can provide important contextual information. Their work differs in many ways. For example, the focus is on the cumulative regret and
the context is observed before the action is taken and cannot be controlled by the learning agent.  


 \citet{Ortega2014thompson} present an analysis and extension of Thompson sampling assuming actions are causal interventions. Their focus is on causal induction (\ie, learning an unknown causal model) instead of exploiting a known causal model. Combining their handling of  causal induction with our analysis is left as future work.

The truncated importance weighted estimators used in \S\ref{sec:simple-regret-general} have been studied before in a causal framework by \citet{BJQ13}, 
where the focus is on learning from observational data, but not controlling the sampling process. They also briefly discuss some of the issues 
encountered in sequential design, but do not give an algorithm or theoretical results for this case.

\section{Problem Setup}
\label{sec:defs}
\newcommand{\bernoulli}{\operatorname{Bernoulli}}
\newcommand{\dirac}{\operatorname{Dirac}}
\renewcommand{\vec}[1]{\boldsymbol{#1}}

We now introduce a novel class of stochastic sequential decision problems which we call \emph{causal bandit problems}. 
In these problems, rewards are given for repeated interventions on a fixed causal model~\cite{Pearl2000}. 
Following the terminology and notation in~\cite{Koller2009}, a \emph{causal model} is given by a directed acyclic graph $\mathcal{G}$ over a set of random variables $\mathcal{X} = \{ X_1, \ldots, X_N \}$ and a joint distribution $\mathrm{P}$ over $\mathcal{X}$ that factorises over $\mathcal{G}$.
We will assume each variable only takes on a finite number of distinct values.
An edge from variable $X_i$ to $X_j$ is interpreted to mean that a change in the value of $X_i$ may directly cause a change to the value of $X_j$.
The \emph{parents} of a variable $X_i$, denoted $\parents{X_i}$, is the set of all variables $X_j$ such that there is an edge from $X_j$ to $X_i$ in $\mathcal{G}$.
An \emph{intervention or action (of size $n$)}, denoted $do(\vec{X}=\vec{x})$, assigns the values $\vec{x}=\{x_1, \ldots, x_n\}$ to the corresponding variables $\vec{X}=\{X_1, \ldots, X_n\} \subset \mathcal{X}$ with the empty intervention (where no variable is set) denoted $do()$.
The intervention also ``mutilates'' the graph $\mathcal{G}$ by removing all edges from $\parents{i}$ to $X_i$ for each $X_i \in \vec{X}$. 
The resulting graph defines a probability distribution $\P{\vec{X}^c | do(\vec{X}=\vec{x})}$ over $\vec{X}^c := \mathcal{X} - \vec{X}$. 
Details can be found in Chapter 21 of~\cite{Koller2009}.

A learner for a casual bandit problem is given the casual model's graph $\mathcal{G}$ and a set of \emph{allowed actions} $\mathcal{A}$.
One variable $Y \in \mathcal{X}$ is designated as the \emph{reward variable} and takes on values in $\{0, 1\}$.
We denote the expected reward for the action $a = do(\vec{X} = \vec{x})$ by $\mu_{a} := \E{Y | do(\vec{X} = \vec{x})}$ and 
the optimal expected reward by $\mu^* := \max_{a\in\actions} \mu_{a}$. 
The causal bandit game proceeds over $T$ rounds.
In round $t$, the learner \emph{intervenes} by choosing $a_t = do(\vec{X}_t = \vec{x}_t) \in \mathcal{A}$ based on previous observations. 
It then \emph{observes} sampled values for all non-intervened variables $\vec{X}^c_t$ drawn from $\P{\vec{X}^c_t | do(\vec{X}_t = \vec{x}_t)}$, 
including the \emph{reward} $Y_t \in \{0,1\}$. 
After $T$ observations the learner outputs an estimate of the optimal action $\hat a^*_T \in \actions$ based on its prior observations.

The objective of the learner is to minimise the simple regret $\simpleregret = \mu^* - \E{\mu_{\hat a^*_T}}.$ This is sometimes refered to as a ``pure exploration''~\citep{Bubeck2009a} or ``best-arm identification'' problem~\citep{Gabillon2012a} and is most appropriate when, as in drug and policy testing, the learner has a fixed experimental budget after which its policy will be fixed indefinitely. 

Although we will focus on the intervene-then-observe ordering of events within each round, other scenarios are possible. If the non-intervened variables are observed before an intervention is selected our framework reduces to stochastic contextual bandits, which are already reasonably well understood~\citep{Agarwal2014}. Even if no observations are made during the rounds, the causal model may still allow offline pruning of the set of allowable interventions thereby reducing the complexity.

We note that classical $K$-armed stochastic bandit problem can be recovered in our framework by considering a simple causal model with one edge connecting a single variable $X$ that can take on $K$ values to a reward variable $Y \in \set{0,1}$ where $\P{Y = 1|X} = r(X)$ for some arbitrary but unknown, real-valued function $r$. The set of allowed actions in this case is $\mathcal{A} = \{ do(X = k) \colon k \in \{1, \ldots, K\}\}$. Conversely, any causal bandit problem can be reduced to a classical stochastic $|\mathcal{A}|$-armed bandit problem by treating each possible intervention as an independent arm and ignoring all sampled values for the observed variables except for the reward. Intuitively though, one would expect to perform better by making use of the extra structure and observations.

\section{Regret Bounds for Parallel Bandit}
\label{sec:simple-regret}
In this section we propose and analyse an algorithm for achieving the optimal regret in a natural special 
case of the causal bandit problem which we call the {\it parallel bandit}.
It is simple enough to admit a thorough analysis but rich enough to model the type of problem discussed in \S\ref{sec:intro}, including the farming example. 
It also suffices to witness the regret gap between algorithms that make use of causal models and those which do not.

The causal model for this class of problems has $N$ binary variables $\{ X_1, \ldots, X_N \}$ where each $X_i \in \{0,1\}$ are independent causes of a 
reward variable $Y \in \set{0,1}$, as shown in Figure~\ref{fig:parallel}.
All variables are observable and the set of allowable actions are all size 0 and size 1 interventions: $\mathcal{A} = \set{do()} \cup \set{ do(X_i = j) \colon 1 \leq i \leq N \text{ and } j \in \set{0,1}}$
In the farming example from the introduction, $X_1$ might represent temperature (\eg, $X_1=0$ for low and $X_1=1$ for high). 
The interventions $do(X_1 = 0)$ and $do(X_1 = 1)$ indicate the use of shades or heat lamps to keep the temperature low or high, respectively.

\begin{figure}
    \begin{subfigure}[b]{0.34\textwidth}
	\centering    
          \begin{tikzpicture}[->,>=stealth',shorten >=1pt,auto,node distance=.45cm,
  thick,main node/.style={observed}, hidden/.style={empty},background rectangle/.style={fill=olive!45}]
\node[main node](1){$X_{1}$};
\node[main node, right=of 1](2){$X_{2}$};
\node[hidden, right=of 2](3){$...$};
\node[main node, right=of 3](4){$X_{N}$};
\node[main node, below right=of 2](5){$Y$};
 \path[every node/.style={font=\tiny}]
    (1) edge (5)
    	(2) edge (5)
    (4) edge (5);
\end{tikzpicture}
        \caption{Parallel graph}
        \label{fig:parallel}
    \end{subfigure}
    \begin{subfigure}[b]{0.2\textwidth}
    \centering
        \begin{tikzpicture}[->,>=stealth',shorten >=1pt,auto,node distance=.45cm,
  thick,main node/.style={observed}, hidden/.style={empty},background rectangle/.style={fill=olive!45}]
\node[main node](1){$X_1$};
\node[main node, below left=of 1](2){$X_2$};
\node[main node, below right=of 1](4){$Y$};
 \path[every node/.style={font=\tiny}]
    (1) edge (2)
    (1) edge (4)
    (2) edge (4);
\end{tikzpicture}
        \caption{Confounded graph}
        \label{fig:causalStructure_confounded}
    \end{subfigure}
    \begin{subfigure}[b]{0.4\textwidth}
    \centering
         \begin{tikzpicture}[->,>=stealth',shorten >=1pt,auto,node distance=.45cm,
  thick,main node/.style={observed}, hidden/.style={empty},background rectangle/.style={fill=olive!45}]
\node[main node](1){$X_{1}$};
\node[main node, right=of 1](2){$X_{2}$};
\node[hidden, right=of 2](3){$...$};
\node[main node, right=of 3](4){$X_{N}$};
\node[main node, right=of 4](5){$Y$};
 \path[every node/.style={font=\tiny}]
    (1) edge (2)
  	(2) edge (3)
    (3) edge (4)
    (4) edge (5);
\end{tikzpicture}
        \caption{Chain graph}
        \label{fig:causalchain}
    \end{subfigure}
    \caption{Causal Models}\label{fig:causalmodels}
\end{figure}
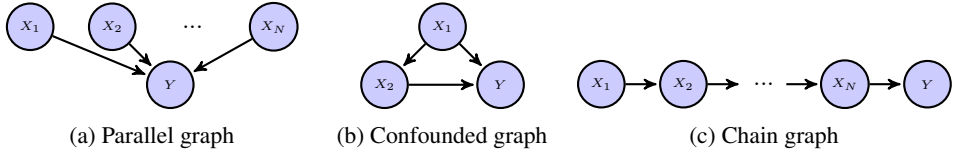

In each round the learner either purely observes by selecting $do()$ or sets the value of a single variable. The remaining variables are simultaneously set by independently biased coin flips. 
The value of all variables are then used to determine the distribution of rewards for that round.
Formally, when not intervened upon we assume that each $X_i \sim \bernoulli(q_i)$ where $\vec{q} = (q_1, \ldots, q_N) \in [0,1]^N$ so that $q_i = \P{X_i = 1}$.
The value of the reward variable is distributed as $\P{Y = 1|\vec{X}} = r(\vec{X})$ where 
$r : \{0,1\}^N \to [0,1]$ is an arbitrary, fixed, and unknown function. 
In the farming example, this choice of $Y$ models the success or failure of a seasons crop, which depends stochastically on the various environment variables.


\paragraph{The Parallel Bandit Algorithm}

The algorithm operates as follows. For the first $T/2$ rounds it chooses $do()$ to collect observational data. As the only link from each $X_1,\ldots,X_N$ to $Y$ is a direct, causal one, $\P{Y|do(X_i=j)}=\P{Y|X_i=j}$. Thus we can create good estimators for the returns of the actions $do(X_i = j)$ for which $\P{X_i = j}$ is large. The actions for which $\P{X_i = j}$ is small may not be observed (often) so  estimates of their returns could be poor. To address this, the remaining $T/2$ rounds are evenly split to estimate the rewards for these infrequently observed actions. The difficulty of the problem depends on $\vec{q}$ and, in particular, how many of the variables are unbalanced (\ie, small $q_i$ or $(1-q_i)$). For $\tau \in [2...N]$ let $I_\tau = \set{ i : \min\set{q_i, 1-q_i} < \frac{1}{\tau}}$. Define
\eq{
\label{eq:m-simple}
m(\vec{q}) = \min \set{ \tau : |I_{\tau}| \leq \tau}\,.
}

\begin{wrapfigure}[18]{r}{0.6\textwidth}
\vspace{-25pt}
\begin{minipage}{.6\textwidth}
\begin{algorithm}[H]
\caption{Parallel Bandit Algorithm}\label{alg:simple}
\begin{algorithmic}[1]
\STATE {\bf Input:} Total rounds $T$ and $N$.
\FOR{$t \in 1,\ldots,T / 2$}
\STATE Perform empty intervention $do()$
\STATE Observe $\vec{X}_t$ and $Y_t$
\ENDFOR
\FOR{$a = do(X_i = x) \in \actions$}
\STATE Count times $X_i = x$ seen: $T_a = \sum_{t=1}^{T/2} \ind{X_{t,i} = x}$
\STATE Estimate reward: $\hat{\mu}_a = \frac{1}{T_a} \sum_{t=1}^{T/2} \ind{X_{t,i} = x} Y_t$ \\[0.2cm]
\STATE Estimate probabilities: $\hat{p}_a = \frac{2 T_a}{T}$,\,\, $\hat q_i = \hat p_{do(X_i = 1)}$
\ENDFOR
\STATE Compute $\hat{m} = m(\vec{\hat q})$ and $A = \set{a \in \actions \colon \hat{p}_a \leq \frac{1}{\hat m}}$.
\STATE Let $T_A := \frac{T}{2 |A|}$ be times to sample each $a\in A$.
\FOR{$a = do(X_i = x) \in A$}
\FOR{$t \in 1,\ldots,T_A$}
\STATE Intervene with $a$ and observe $Y_t$
\ENDFOR
\STATE Re-estimate $\hat{\mu}_a = \frac{1}{T_A} \sum_{t=1}^{T_A} Y_t$
\ENDFOR
\RETURN estimated optimal $\hat{a}^*_T \in \argmax_{a\in\actions} \hat{\mu}_a$
\end{algorithmic}
\end{algorithm}
\end{minipage}
\end{wrapfigure}

$I_\tau$ is the set of variables considered unbalanced and we tune $\tau$ to trade off identifying the low probability actions against not having too many of them, so as to minimize the worst-case simple regret. When $\vec{q} = (\frac{1}{2}, \ldots, \frac{1}{2})$ we have $m(\vec{q}) = 2$ and when $\vec{q} = (0, \ldots, 0)$ we have $m(\vec{q}) = N$. We do not assume that $\vec{q}$ is known, thus Algorithm \ref{alg:simple} also utilizes the samples captured during the observational phase to estimate $m(\vec{q})$. Although very simple, the following two theorems show that this algorithm is effectively optimal.

\begin{theorem}\label{thm:uq-simple}
Algorithm \ref{alg:simple} satisfies
\eq{
\simpleregret \in \bigo{\sqrt{\frac{m(\vec{q})}{T}\log\left(\frac{NT}{m}\right)}}\,.
}
\end{theorem}

\begin{theorem}\label{thm:lower}
For all $T$, $\vec{q}$ and all strategies, there exists a reward function such that
\eq{
\simpleregret 
\in \Omega\left(\sqrt{\frac{m(\vec{q})}{T}}\right)\,.
}
\end{theorem}

\ifsup
The proofs of Theorems \ref{thm:uq-simple} and \ref{thm:lower} may be found in Sections \ref{sec:thm:uq-simple} and \ref{sec:thm:lower} respectively.
\fi
By utilizing knowledge of the causal structure, Algorithm \ref{alg:simple} effectively only has to explore the $m(\vec{q})$ 'difficult' actions. Standard multi-armed bandit algorithms must explore all $2N$ actions and thus achieve regret  $\smash{\Omega(\sqrt{N/T})}$. Since $m$ is typically much smaller than $N$, the new algorithm can significantly outperform classical bandit algorithms in this setting. In practice, you would combine the data from both phases to estimate rewards for the low probability actions. We do not do so here as it slightly complicates the proofs and does not improve the worst case regret.

\section{Regret Bounds for General Graphs}
\label{sec:simple-regret-general}
We now consider the more general problem where the graph structure is known, but arbitrary. For general graphs, $\P{Y|X_i=j} \neq \P{Y|do(X_i=j)}$ (correlation is not causation). However, if all the variables are observable, any causal distribution $\P{X_1...X_N|do(X_i=j)}$ can be expressed in terms of observational distributions via the truncated factorization formula \cite{Pearl2000}. 
\eq{
\P{X_1...X_N|do(X_i=j)} = 
\prod_{k \neq i}\P{X_k|\parents{X_k}}\delta(X_i - j)\,, 
} 
where $\parents{X_k}$ denotes the parents of $X_k$ and $\delta$ is the dirac delta function. 

We could naively generalize our approach for parallel bandits by observing for $T/2$ rounds, applying the truncated product factorization to 
write an expression for each $\P{Y|a}$ in terms of observational quantities and explicitly playing the actions for which the observational 
estimates were poor. However, it is no longer optimal to ignore the information we can learn about the reward for intervening on one variable 
from rounds in which we act on a different variable. Consider the graph in Figure \ref{fig:causalchain} and suppose each variable deterministically 
takes the value of its parent, $X_k = X_{k-1}$ for $k\in {2,\ldots,N}$ and $\P{X_1} = 0$. We can learn the reward for all the interventions $do(X_i = 1)$ 
simultaneously by selecting $do(X_1 = 1)$, but not from $do()$. In addition, variance of the observational estimator for $a = do(X_i = j)$ can be 
high even if $\P{X_i = j}$ is large. Given the causal graph in Figure \ref{fig:causalStructure_confounded}, $\P{Y|do(X_2= j)} = \sum_{X_1}\P{X_1}\P{Y|X_1, X_2 = j}$. 
Suppose $X_2 = X_1$ deterministically, no matter how large $\P{X_2 = 1}$ is we will never observe $(X_2=1,X_1 = 0)$ and so cannot 
get a good estimate for $\P{Y|do(X_2=1)}$. 

To solve the general problem we need an estimator for each action that incorporates information obtained from every other action and a way to optimally 
allocate samples to actions. To address this difficult problem, we assume the conditional interventional distributions $\P{\parents{Y}|a}$ (but not $\P{Y|a}$) 
are known. These could be estimated from experimental data on the same covariates but where the outcome of interest differed, such that $Y$ was not included, 
or similarly from observational data subject to identifiability constraints. Of course this is a somewhat limiting assumption, but seems like a natural place to
start. The challenge of estimating the conditional distributions for all variables in an optimal way is left as an interesting future direction.
Let $\eta$ be a distribution on available interventions $a \in \calA$ so $\eta_a \geq 0$ and $\sum_{a \in \calA} \eta_a = 1$. 
Define $Q = \sum_{a \in \calA} \eta_a \P{\parents{Y}|a}$ to be the mixture distribution over the interventions with respect to $\eta$.

\begin{wrapfigure}[10]{r}{0.5\textwidth}
\vspace{-30pt}
\begin{minipage}{.5\textwidth}
\begin{algorithm}[H]
\caption{General Algorithm}\label{alg:general}
\begin{algorithmic}
\STATE {\bf Input:} $T$, $\eta \in [0,1]^{\calA}$, $B \in [0,\infty)^{\calA}$
\FOR{$t \in \set{1,\ldots,T}$}
\STATE Sample action $a_t$ from $\eta$
\STATE Do action $a_t$ and observe $X_t$ and $Y_t$
\ENDFOR
\FOR{$a \in \calA$}
\STATE
\eq {
\hat \mu_a =  \frac{1}{T} \sum_{t=1}^T Y_t R_a(X_t)  \ind{R_a(X_t) \leq B_a}
}
\ENDFOR
\STATE {\bf return} $\hat a^*_T = \argmax_a \hat \mu_a$
\end{algorithmic}
\end{algorithm}
\end{minipage}
\end{wrapfigure}

Our algorithm samples $T$ actions from $\eta$ and uses them to estimate the returns $\mu_a$ for all $a \in \calA$ simultaneously via a truncated importance weighted estimator. Let $\parents{Y}(X)$ denote the realization of the variables in $X$ that are parents of Y and define $R_a(X) = \frac{\Pn{a}{\parents{Y}(X)}}{\Q{\parents{Y}(X)}}$

\eq {
\hat \mu_a =  \frac{1}{T} \sum_{t=1}^T Y_t R_a(X_t)  \ind{R_a(X_t) \leq B_a}\,, 
} 

where $ B_a \geq 0$  is a constant that tunes the level of truncation to be chosen subsequently. The truncation introduces a bias in the estimator, but simultaneously chops the potentially heavy tail that is so detrimental to its concentration guarantees. 

The distribution over actions, $\eta$ plays the role of allocating samples to actions and is optimized to minimize the worst-case simple regret. Abusing notation we define $m(\eta)$ by
\eq{
m(\eta) = \max_{a \in \calA} \EEa\left[\frac{\Pn{a}{\parents{Y}(X)}}{\Q{\parents{Y}(X)}}\right]\,,\text{ where } \EEa \text{ is the expectation with respect to } \Pn{a}.
}

We will show shortly that $m(\eta)$ is a measure of the difficulty of the problem that approximately coincides with the version for parallel bandits, justifying the name overloading.

\begin{theorem}\label{thm:general}
If Algorithm \ref{alg:general} is run with $B \in \R^{\calA}$ given by $B_a = \sqrt{\frac{m(\eta)T}{\log\left(2T|\calA|\right)}}\,.$

\eq{
\simpleregret \in \bigo{\sqrt{\frac{m(\eta)}{T} \log\left(2T|\calA|\right)}}\,.
}
\end{theorem}
\ifsup 
The proof is in Section \ref{sec:thm:general}.
\fi Note the regret has the same form as that obtained for Algorithm \ref{alg:simple}, with $m(\eta)$ replacing $m(q)$. Algorithm \ref{alg:simple} assumes only the graph structure and not knowledge of the conditional distributions on $X$. Thus it has broader applicability to the parallel graph than the generic algorithm given here. We believe that Algorithm \ref{alg:general} with the optimal choice of $\eta$ is close to minimax optimal, but leave lower bounds
for future work.

\paragraph{Choosing the Sampling Distribution} Algorithm \ref{alg:general} depends on a choice of sampling distribution $\operatorname{Q}$ that is determined by $\eta$. In light of Theorem \ref{thm:general}
a natural choice of $\eta$ is the minimiser of $m(\eta)$.
\eq{
\eta^* 
= \argmin_\eta m(\eta) = \argmin_\eta \underbrace{\max_{a \in \calA} \EEa \left[\frac{\Pn{a}{\parents{Y}(X)}}{\sum_{b \in \calA} \eta_b \Pn{b}{\parents{Y}(X)}}\right]}_{m(\eta)}\,.
}
Since the mixture of convex functions is convex and the maximum of a set of convex functions is convex, we see that $m(\eta)$ is convex (in $\eta$).
Therefore the minimisation problem may be tackled using standard techniques from convex optimisation. An interpretation of $m(\eta^*)$ is the minimum achievable worst-case variance of the importance weighted estimator. In the experimental section we present some special cases, but for now we give two simple results. The first shows that $|\calA|$ serves as an upper bound on $m(\eta^*)$.

\begin{proposition}\label{pro:m-bound}
$m(\eta^*) \leq |\calA|$. \textit{Proof.} 
\textup{By definition, $m(\eta^*) \leq m(\eta)$ for all $\eta$. Let $\eta_a = 1/|\calA|\,\forall a$.}
\eq{
m(\eta) 
= \max_a \EEa\left[\frac{\Pn{a}{\parents{Y}(X)}}{\Q{\parents{Y}(X)}}\right] 
\leq \max_a \EEa\left[\frac{\Pn{a}{\parents{Y}(X)}}{\eta_a \Pn{a}{\parents{Y}(X)}}\right] 
= \max_a \EEa\left[\frac{1}{\eta_a}\right] = |\calA| 
}
\end{proposition} 

The second observation is that, in the parallel bandit setting, $m(\eta^*) \leq 2m(\boldsymbol{q})$. This is easy to see by letting $\eta_a = 1/2$ for $a = do()$ and $\eta_a = \ind{\P{X_i = j} \leq 1/m(\boldsymbol{q})} / 2m(\boldsymbol{q})$ for the actions corresponding to $do(X_i=j)$, and applying an argument like that for Proposition~\ref{pro:m-bound}. \ifsup 
The proof is in Section \ref{sec:m-equivelence}.
\fi

\begin{remark}\label{rem:truncate}
The choice of $B_a$ given in Theorem \ref{thm:general} is not the only possibility. As we shall see in the experiments, it is 
often possible to choose $B_a$ significantly
larger when there is no heavy tail and this can drastically improve performance by eliminating the bias. This is especially true when the ratio $R_a$ is never too large
and Bernstein's inequality could be used directly without the truncation. For another discussion see the article by \citet{BJQ13} who also use importance weighted estimators
to learn from observational data.
\end{remark}

\section{Experiments}
\label{sec:experiments}

We compare Algorithms \ref{alg:simple} and \ref{alg:general} with Successive Elimination on the parallel bandit problem
under a variety of conditions, including where the importance weighted estimator used by Algorithm \ref{alg:general} is not truncated,
which is justified in this setting by Remark \ref{rem:truncate}. Throughout we use a model in which $Y$ depends only on a single variable $X_1$ (this is unknown to the algorithms). $Y_t \sim \bernoulli(\frac{1}{2}+\epsilon)$ if $X_1=1$ and $Y_t \sim \bernoulli(\frac{1}{2}-\epsilon')$ otherwise, where $\epsilon' = q_1\epsilon/(1-q_1)$. This leads to an expected reward of $\frac{1}{2}+\epsilon$ for $do(X_1=1)$, $\frac{1}{2}-\epsilon'$ for $do(X_1=0)$ and $\frac{1}{2}$ for all other actions. We set $q_i = 0$ for $i \leq m$ and $\frac{1}{2}$ otherwise. Note that changing $m$ and thus $\boldsymbol{q}$ has no effect on the reward distribution. 


We compare the performance of the Algorithm 1, which is specific to the parallel problem, but does not require knowledge of $\boldsymbol{q}$, with that of Algorithm 2 and the Successive Reject algorithm of \cite{audibert2010best}. For each experiment, we show the average regret over 10,000 simulations with error bars displaying three standard errors.

In figure \ref{fig:simple_vs_m} we fix the number of variables $N$ and the horizon $T$ and compare the performance of the algorithms as $m$ increases. The regret for the Successive Reject algorithm is constant as it depends only on the reward distribution and has no knowledge of the causal structure. For the causal algorithms it increases approximately with $\sqrt{m}$. As $m$ approaches $N$, the gain the causal algorithms obtain from knowledge of the structure is outweighed by fact they do not leverage the observed rewards to focus sampling effort on actions with high pay-offs.

\begin{figure}
    \begin{subfigure}[t]{0.3\textwidth}
		\centering    
    		\includegraphics[width=\textwidth]{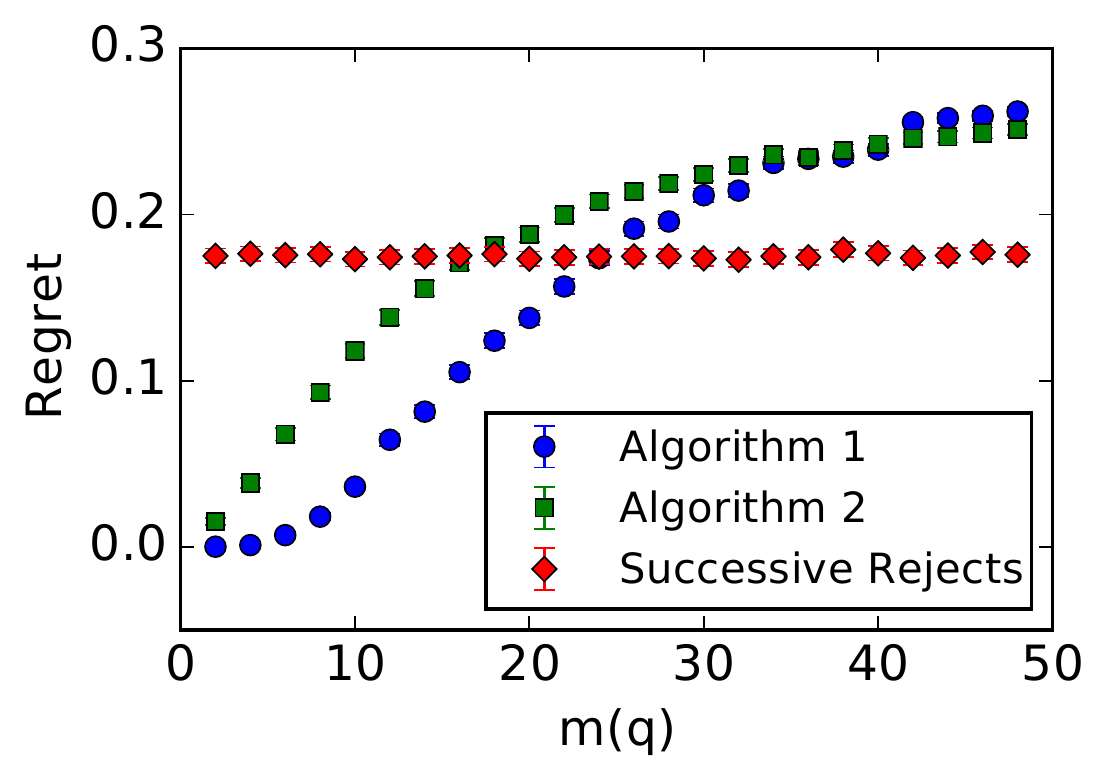}
    		\caption{Simple regret vs $m(\boldsymbol{q})$ for fixed horizon $T=400$ and number of variables $N = 50$}
        \label{fig:simple_vs_m}
    \end{subfigure}\hfill
    \begin{subfigure}[t]{0.3\textwidth}
    		\centering
        \includegraphics[width=\textwidth]{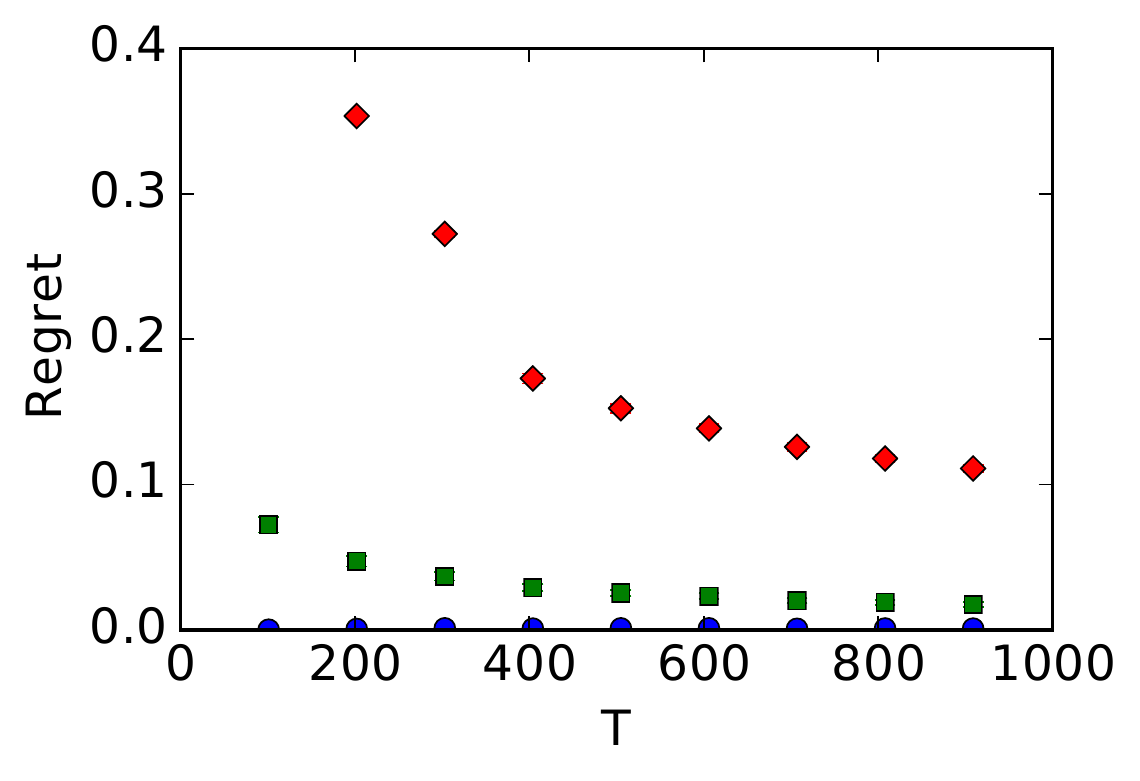}
    		\caption{Simple regret vs horizon, $T$, with $N = 50$, $m=2$ and $\epsilon = \sqrt{\frac{N}{8T}}$}
        \label{fig:simple_vs_T_vary_epsilon}
    \end{subfigure}\hfill
    \begin{subfigure}[t]{0.3\textwidth}
    		\centering
    		\includegraphics[width=\textwidth]{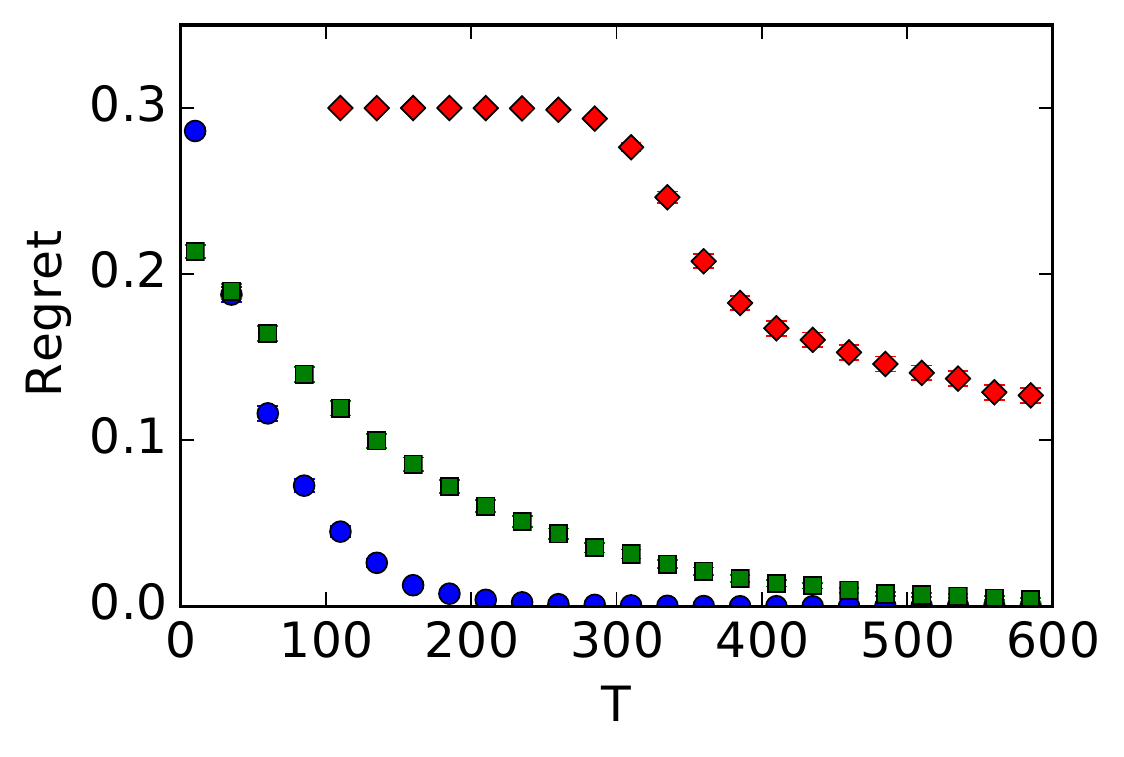}
    		\caption{Simple regret vs horizon, $T$, with $N = 50$, $m=2$ and fixed $\epsilon = .3$}
    		\label{fig:simple_vs_T}
    \end{subfigure}
    \caption{Experimental results}
    \label{fig:experiments}
\end{figure}


Figure \ref{fig:simple_vs_T_vary_epsilon} demonstrates the performance of the algorithms in the worst case environment for standard bandits, where the gap between the optimal and sub-optimal arms, $\smash{\epsilon = \sqrt{N/(8T)}}$ , is just too small to be learned. This gap is learn-able by the causal algorithms, for which the worst case $\epsilon$ depends on $m \ll N$. In figure \ref{fig:simple_vs_T} we fix $N$ and $\epsilon$ and observe that, for sufficiently large $T$, the regret decays exponentially. The decay constant is larger for the causal algorithms as they have observed a greater effective number of samples for a given $T$. 

For the parallel bandit problem, the regression estimator used in the specific algorithm outperforms the truncated importance weighted estimator in the more general algorithm, despite the fact the specific algorithm must estimate $\boldsymbol{q}$ from the data. 
This is an interesting phenomenon that has been noted before in off-policy evaluation where the regression (and not the importance weighted) estimator is known to be minimax optimal asymptotically \citep{LMS14}.

\section{Discussion \& Future Work}
\label{sec:discussion}
Algorithm~\ref{alg:general} for general causal bandit problems 
estimates the reward for all allowable interventions $a \in \calA$ over $T$ rounds by sampling and applying interventions from a distribution $\eta$.
Theorem~\ref{thm:general} shows that this algorithm has (up to log factors) simple regret that is $\smash{\mathcal O(\sqrt{m(\eta)/T)}}$ where 
the parameter $m(\eta)$ measures the difficulty of learning the causal model and is always less than $N$.
The value of $m(\eta)$ is a uniform bound on the variance of the reward estimators $\hat{\mu}_a$ and, intuitively, problems where all variables' values in the causal model ``occur naturally'' when interventions are sampled from $\eta$ will have low values of $m(\eta)$.

The main practical drawback of Algorithm~\ref{alg:general} is that both the estimator $\hat{\mu}_a$ and the optimal sampling distribution $\eta^*$ (\ie, the one that minimises $m(\eta)$) require knowledge of the conditional distributions $\Pn{a}{\parents{Y}}$ for all $a \in \calA$.
In contrast, in the special case of parallel bandits, Algorithm~\ref{alg:simple} uses the $do()$ action to effectively estimate $m(\eta)$ and the rewards then re-samples the interventions with variances that are not bound by $\hat{m}(\eta)$.
Despite these extra estimates, Theorem~\ref{thm:lower} shows that this approach is optimal (up to log factors).
Finding an algorithm that only requires the causal graph and lower bounds for its simple regret in the general case is left as future work.

\vspace{-0.4cm} 
\paragraph{Making Better Use of the Reward Signal}
Existing algorithms for best arm identification are based on ``successive rejection'' (SR) of arms based on UCB-like bounds on their rewards~\citep{Even-Dar2002}.
In contrast, our algorithms completely ignore the reward signal when developing their arm sampling policies and only use the rewards when estimating $\hat{\mu}_a$.
Incorporating the reward signal into our sampling techniques or designing more adaptive reward estimators that focus on high reward interventions is an obvious next step.
This would likely improve the poor performance of our causal algorithm relative to the sucessive rejects algorithm for large $m$, as seen in Figure~\ref{fig:simple_vs_m}.
For the parallel bandit the required modifications should be quite straightforward. The idea would be to adapt the algorithm to essentially use successive elimination in
the second phase so arms are eliminated as soon as they are provably no longer optimal with high probability. In the general case a similar modification is also possible
by dividing the budget $T$ into phases and optimising the sampling distribution $\eta$, eliminating arms when their confidence intervals are no longer overlapping. Note
that these modifications will not improve the minimax regret, which at least for the parallel bandit is already optimal. For this reason we prefer to emphasize 
the main point that causal structure should be exploited when available. Another observation is that Algorithm \ref{alg:general} is actually using a fixed design, which
in some cases may be preferred to a sequential design for logistical reasons. This is not possible for Algorithm \ref{alg:simple}, since the $\vec{q}$ vector is unknown.

\vspace{-0.4cm} 
\paragraph{Cumulative Regret}
Although we have focused on simple regret in our analysis, it would also be natural to consider the cumulative regret.
In the case of the parallel bandit problem we can slightly modify the analysis from \citep{wu2015online} on bandits with side information 
to get near-optimal cumulative regret guarantees. They consider a finite-armed bandit model with side information where in reach round
the learner chooses an action and receives a Gaussian reward signal for all actions, but with a known variance that depends on the chosen action.
In this way the learner can gain information about actions it does not take with varying levels of accuracy. The reduction follows by substituting
the importance weighted estimators in place of the Gaussian reward. In the case that $\vec{q}$ is known this would lead to a known variance
and the only (insignificant) difference is the Bernoulli noise model. In the parallel bandit case we believe this would lead to near-optimal cumulative regret,
at least asymptotically. 


The parallel bandit problem can also be viewed as an instance of a time varying graph feedback problem \citep{Alon2015,Kocak2014}, where at each timestep the feedback graph $G_t$ is selected stochastically, dependent on $\boldsymbol{q}$, and revealed after an action has been chosen. The feedback graph is distinct from the causal graph. A link $A \rightarrow B$ in $G_t$ indicates that selecting the action $A$ reveals the reward for action $B$. For this parallel bandit problem, $G_t$ will always be a star graph with the action $do()$ connected to half the remaining actions. However, \citet{Alon2015,Kocak2014} give adversarial algorithms, which when applied to the parallel bandit problem obtain the standard bandit regret. A malicious adversary can select the same graph each time, such that the rewards for half the arms are never revealed by the informative action. This is equivalent to a nominally stochastic selection of feedback graph where $\boldsymbol{q} = \boldsymbol{0}$. 


\vspace{-0.4cm} 
\paragraph{Causal Models with Non-Observable Variables}
If we assume knowledge of the conditional \textit{interventional} distributions $\Pn{a}{\parents{Y}}$ our analysis applies unchanged to the case of causal models with 
non-observable variables. Some of the interventional distributions may be non-identifiable meaning we can not obtain prior estimates for $\Pn{a}{\parents{Y}}$ from 
even an infinite amount of observational data. Even if all variables are observable and the graph is known, if the conditional distributions are unknown, then Algorithm
\ref{alg:general} cannot be used. Estimating these quantities while simultaneously minimising the simple regret is an interesting and challenging open problem.



\vspace{-0.4cm} 
\paragraph{Partially or Completely Unknown Causal Graph}
A much more difficult generalisation would be to consider causal bandit problems where the causal graph is completely unknown or known to be a member of class of models.
The latter case arises naturally if we assume free access to a large observational dataset, from which the Markov equivalence class can be found via causal discovery techniques. 
Work on the problem of selecting experiments to discover the correct causal graph from within a Markov equivalence class~\cite{Eberhardt2005,eberhardt2010causal,hauser2014two,Hu2014b} could potentially be incorporated into a causal bandit algorithm.
In particular, \citet{Hu2014b} show that only $\bigo{\log \log n}$ multi-variable interventions are required on average to recover a causal graph over $n$ variables once purely observational data is used to recover the ``essential graph''.
Simultaneously learning a completely unknown causal model while estimating the rewards of interventions without a large observational dataset would be much more challenging.



{\small\bibliography{libraryicml}}

\begin{thebibliography}{}

\bibitem[Agarwal et~al., 2014]{Agarwal2014}
Agarwal, A., Hsu, D., Kale, S., Langford, J., Li, L., and Schapire, R.~E.
  (2014).
\newblock Taming the monster: A fast and simple algorithm for contextual
  bandits.
\newblock In {\em ICML}, pages 1638--1646.

\bibitem[Alon et~al., 2015]{Alon2015}
Alon, N., Cesa-Bianchi, N., Dekel, O., and Koren, T. (2015).
\newblock Online learning with feedback graphs: Beyond bandits.
\newblock In {\em COLT}, pages 23--35.

\bibitem[Audibert and Bubeck, 2010]{audibert2010best}
Audibert, J.-Y. and Bubeck, S. (2010).
\newblock Best arm identification in multi-armed bandits.
\newblock In {\em COLT}, pages 13--p.

\bibitem[Auer et~al., 1995]{Auer1995}
Auer, P., Cesa-Bianchi, N., Freund, Y., and Schapire, R. (1995).
\newblock {Gambling in a rigged casino: The adversarial multi-armed bandit
  problem}.
\newblock {\em Proceedings of IEEE 36th Annual Foundations of Computer
  Science}, pages 322--331.

\bibitem[Avner et~al., 2012]{Avner2012}
Avner, O., Mannor, S., and Shamir, O. (2012).
\newblock Decoupling exploration and exploitation in multi-armed bandits.
\newblock In {\em ICML}, pages 409--416.

\bibitem[Bareinboim et~al., 2015]{Bareinboim2015}
Bareinboim, E., Forney, A., and Pearl, J. (2015).
\newblock Bandits with unobserved confounders: A causal approach.
\newblock In {\em NIPS}, pages 1342--1350.

\bibitem[Bart{\'{o}}k et~al., 2014]{Bartok2014}
Bart{\'{o}}k, G., Foster, D.~P., P{\'{a}}l, D., Rakhlin, A., and
  Szepesv{\'{a}}ri, C. (2014).
\newblock {Partial monitoring-classification, regret bounds, and algorithms}.
\newblock {\em Mathematics of Operations Research}, 39(4):967--997.

\bibitem[Bottou et~al., 2013]{BJQ13}
Bottou, L., Peters, J., Quinonero-Candela, J., Charles, D.~X., Chickering,
  D.~M., Portugaly, E., Ray, D., Simard, P., and Snelson, E. (2013).
\newblock Counterfactual reasoning and learning systems: The example of
  computational advertising.
\newblock {\em JMLR}, 14(1):3207--3260.

\bibitem[Bubeck et~al., 2009]{Bubeck2009a}
Bubeck, S., Munos, R., and Stoltz, G. (2009).
\newblock Pure exploration in multi-armed bandits problems.
\newblock In {\em ALT}, pages 23--37.

\bibitem[Chernoff, 1959]{Chernoff1959}
Chernoff, H. (1959).
\newblock Sequential design of experiments.
\newblock {\em The Annals of Mathematical Statistics}, pages 755--770.

\bibitem[Eberhardt, 2010]{eberhardt2010causal}
Eberhardt, F. (2010).
\newblock {Causal Discovery as a Game.}
\newblock In {\em NIPS Causality: Objectives and Assessment}, pages 87--96.

\bibitem[Eberhardt et~al., 2005]{Eberhardt2005}
Eberhardt, F., Glymour, C., and Scheines, R. (2005).
\newblock {On the number of experiments sufficient and in the worst case
  necessary to identify all causal relations among n variables}.
\newblock In {\em UAI}.

\bibitem[Even-Dar et~al., 2002]{Even-Dar2002}
Even-Dar, E., Mannor, S., and Mansour, Y. (2002).
\newblock Pac bounds for multi-armed bandit and markov decision processes.
\newblock In {\em Computational Learning Theory}, pages 255--270.

\bibitem[Gabillon et~al., 2012]{Gabillon2012a}
Gabillon, V., Ghavamzadeh, M., and Lazaric, A. (2012).
\newblock Best arm identification: A unified approach to fixed budget and fixed
  confidence.
\newblock In {\em NIPS}, pages 3212--3220.

\bibitem[Hagerup and R{\"u}b, 1990]{hagerup1990guided}
Hagerup, T. and R{\"u}b, C. (1990).
\newblock A guided tour of chernoff bounds.
\newblock {\em Information processing letters}, 33(6):305--308.

\bibitem[Hauser and B{\"{u}}hlmann, 2014]{hauser2014two}
Hauser, A. and B{\"{u}}hlmann, P. (2014).
\newblock {Two optimal strategies for active learning of causal models from
  interventional data}.
\newblock {\em International Journal of Approximate Reasoning}, 55(4):926--939.

\bibitem[Hu et~al., 2014]{Hu2014b}
Hu, H., Li, Z., and Vetta, A.~R. (2014).
\newblock Randomized experimental design for causal graph discovery.
\newblock In {\em NIPS}, pages 2339--2347.

\bibitem[Jamieson et~al., 2014]{Jamieson2013}
Jamieson, K., Malloy, M., Nowak, R., and Bubeck, S. (2014).
\newblock lil'{UCB}: An optimal exploration algorithm for multi-armed bandits.
\newblock In {\em COLT}, pages 423--439.

\bibitem[Koc{\'a}k et~al., 2014]{Kocak2014}
Koc{\'a}k, T., Neu, G., Valko, M., and Munos, R. (2014).
\newblock Efficient learning by implicit exploration in bandit problems with
  side observations.
\newblock In {\em NIPS}, pages 613--621.

\bibitem[Koller and Friedman, 2009]{Koller2009}
Koller, D. and Friedman, N. (2009).
\newblock {\em {Probabilistic graphical models: principles and techniques}}.
\newblock MIT Press.

\bibitem[Langford and Zhang, 2008]{Langford2008}
Langford, J. and Zhang, T. (2008).
\newblock The epoch-greedy algorithm for multi-armed bandits with side
  information.
\newblock In {\em NIPS}, pages 817--824.

\bibitem[Li et~al., 2014]{LMS14}
Li, L., Munos, R., and Szepesvari, C. (2014).
\newblock On minimax optimal offline policy evaluation.
\newblock {\em arXiv preprint arXiv:1409.3653}.

\bibitem[Ortega and Braun, 2014]{Ortega2014thompson}
Ortega, P.~A. and Braun, D.~A. (2014).
\newblock Generalized thompson sampling for sequential decision-making and
  causal inference.
\newblock {\em Complex Adaptive Systems Modeling}, 2(1):2.

\bibitem[Pearl, 2000]{Pearl2000}
Pearl, J. (2000).
\newblock {\em {Causality: models, reasoning and inference}}.
\newblock MIT Press, Cambridge.

\bibitem[Robbins, 1952]{Robbins1952}
Robbins, H. (1952).
\newblock {Some aspects of the sequential design of experiments}.
\newblock {\em Bulletin of the American Mathematical Society}, 58(5):527--536.

\bibitem[Tsybakov, 2008]{Tsy08}
Tsybakov, A.~B. (2008).
\newblock {\em Introduction to nonparametric estimation}.
\newblock Springer Science \& Business Media.

\bibitem[Wu et~al., 2015]{wu2015online}
Wu, Y., Gy{\"{o}}rgy, A., and Szepesv{\'{a}}ri, C. (2015).
\newblock {Online Learning with Gaussian Payoffs and Side Observations}.
\newblock In {\em NIPS}, pages 1360--1368.

\bibitem[Yu and Mannor, 2009]{Yu2009}
Yu, J.~Y. and Mannor, S. (2009).
\newblock Piecewise-stationary bandit problems with side observations.
\newblock In {\em ICML}, pages 1177--1184.

\end{thebibliography}
\bibliographystyle{apalike}

\ifsup

\section{Proof of Theorem \ref{thm:uq-simple}}\label{sec:thm:uq-simple}

Assume without loss of generality that $q_1 \leq q_2 \leq \ldots \leq q_N \leq 1/2$. The assumption is non-restrictive since all variables
are independent and permutations of the variables can be pushed to the reward function.
The proof of Theorem \ref{thm:uq-simple} requires some lemmas. \

\begin{lemma}\label{lem:conc1}
Let $i \in \set{1,\ldots, N}$ and $\delta > 0$. Then
\eq{
\P{\left|\hat q_i - q_i\right| \geq \sqrt{\frac{6q_i}{T} \log \frac{2}{\delta}}} \leq \delta\,.
}
\end{lemma}

\begin{proof}
By definition, $\hat{q}_i = \frac{2}{T}\sum_{t=1}^{T/2}X_{t,i}$, where $X_{t,i} \sim Bernoulli(q_i)$. Therefore from the Chernoff bound (see equation 6 in \cite{hagerup1990guided}),

\eq{
\P{\left|\hat q_i - q_i\right| \geq \epsilon} \leq 2e^{-\frac{T\epsilon^2}{6q_i}}
}

Letting $\delta = 2e^{-\frac{T\epsilon^2}{6q_i}}$ and solving for $\epsilon$ completes the proof.

\end{proof}

\begin{lemma}\label{lem:conc2}
Let $X_1,X_2\ldots,$ be a sequence of random variables with $X_i \in [0,1]$ and $\EE[X_i] = p$ and $\delta \in [0,1]$.
Then 
\eq{
\P{\exists t \geq n_0 : \left|\frac{1}{t} \sum_{s=1}^t X_s - p\right| \geq \sqrt{\frac{2}{n_0} \log \frac{2}{\delta}}} \leq 4\delta\,.
}
\end{lemma}

\begin{proof}
For $\delta \geq 1/4$ the result is trivial. Otherwise 
by Hoeffding's bound and the union bound:
\eq{
\P{\exists t \geq n_0 : \left|\frac{1}{t} \sum_{s=1}^t X_s - p\right| \geq \sqrt{\frac{2}{n_0} \log \frac{2}{\delta}}} 
&\leq \sum_{t = n_0}^\infty \P{\left|\frac{1}{t} \sum_{s=1}^t X_s - p\right| \geq \sqrt{\frac{2}{n_0} \log \frac{2}{\delta}}} \\
&\leq 2\sum_{t=n_0}^\infty \exp\left(-\frac{t}{n_0} \log \frac{2}{\delta}\right) 
\leq 4\delta\,. \qedhere
}
\end{proof}

\begin{lemma}\label{lem:m_est}
Let $\delta \in (0,1)$ and assume $T \geq 48m \log\frac{2N}{\delta}$. Then
\eq{
\P{2m(\vec{q}) / 3 \leq m(\vec{\hat q}) \leq 2m(\vec{q})} \geq 1 - \delta\,.
}
\end{lemma}

\begin{proof}
Let $F$ be the event that there exists and $1 \leq i \leq N$ for which
\eq{
\left|\hat q_i - q_i\right| \geq \sqrt{\frac{6q_i}{T} \log \frac{2N}{\delta}}\,.
}
Then by the union bound and Lemma \ref{lem:conc1} we have $\P{F} \leq \delta$. The result will be completed by showing that
when $F$ does not hold we have $2m(\vec{q})/3 \leq m(\vec{\hat q}) \leq 2m(\vec{q})$.
From the definition of $m(\vec{q})$ and our assumption on $\vec{q}$ we have for $i > m$ that $q_i \geq q_m \geq 1/m$ and so by Lemma \ref{lem:conc1} we have
\eq{
\frac{3}{4} 
&\geq \frac{1}{2} + \sqrt{\frac{3}{T} \log \frac{2N}{\delta}} 
\geq q_i + \sqrt{\frac{6q_i}{T} \log \frac{2N}{\delta}} 
\geq \hat q_i \\
&\geq q_i - \sqrt{\frac{6q_i}{T} \log \frac{2N}{\delta}}
\geq q_i - \sqrt{\frac{q_i}{8m}}
\geq \frac{1}{2m}\,.
}
Therefore by the pigeonhole principle we have $m(\vec{\hat q}) \leq 2m$.
For the other direction we proceed in a similar fashion. Since the failure event $F$ does not hold we have for $i \leq m$ that
\eq{
\hat q_i 
\leq q_i + \sqrt{\frac{6q_i}{T} \log\frac{2N}{\delta}} 
\leq \frac{1}{m} \left(1 + \sqrt{\frac{1}{8}}\right)
\leq \frac{3}{2m}\,.
}
Therefore $m(\vec{\hat q}) \geq 2m(\vec{q}) / 3$ as required. 
\end{proof}

\begin{proof}[Proof of Theorem \ref{thm:uq-simple}]
Let $\delta = m = m(\vec{q}) / N$. Then by Lemma \ref{lem:m_est} we have 
\eq{
\P{2m/3 \leq m(\vec{\hat q}) \leq 2m} \geq 1 - \delta\,.
}
Recall that $A = \set{a \in \actions : \hat p_a \leq 1/m(\vec{\hat q})}$. Then
for $a \in A$ the algorithm estimates $\mu_a$ from $T/(2m(\vec{\hat q})) \geq T/(4m)$ samples.
Therefore by Hoeffding's inequality and the union bound we have
\eq{
\P{\exists a \in A : |\mu_a - \hat \mu_a| \geq \sqrt{\frac{8m}{T} \log\frac{2N}{\delta}}} \leq \delta\,.
}
For arms not in $a$ we have $\hat p_a \geq 1/m(\vec{\hat q}) \geq 1/(2m)$.
Therefore if $a = do(X_i = j)$, then 
\eq{
\hat p_a = \frac{2}{T} \sum_{t=1}^{T/2} \ind{X_i = j} \geq \frac{1}{2m}\,. 
}
Therefore $\sum_{t=1}^{T/2} \ind{X_{t,i} = j} \geq T/4m$
and by Lemma \ref{lem:conc2} we have
\eq{
\P{\sum_{t=1}^{T/2} \ind{X_i = j} \geq \frac{T}{4m} \text{ and } \left|\hat \mu_a - \mu_a\right| \geq \sqrt{\frac{8m}{T} \log \frac{2N}{\delta}}} \leq 4\delta / N\,.
}
Therefore with probability at least $1 - 6\delta$ we have
\eq{
(\forall a \in \actions) \qquad |\hat \mu_a - \mu_a| \leq \sqrt{\frac{8m}{T} \log \frac{N}{\delta}} = \epsilon\,.
}
If this occurs, then 
\eq{
\mu_{\hat a^*_T} \geq \hat \mu_{\hat a^*_T} - \epsilon \geq \hat \mu_{a^*} - \epsilon \geq \mu_{a^*} - 2\epsilon\,.
}
Therefore
\eq{
\mu^* - \EE[\mu_{\hat a^*_T}] 
\leq 6\delta + \epsilon 
\leq \frac{6m}{T} + \sqrt{\frac{32m}{T} \log \frac{NT}{m}}\,, 
}
which completes the result.
\end{proof}

\section{Proof of Theorem \ref{thm:lower}}\label{sec:thm:lower}

We follow a relatively standard path by choosing multiple environments that have different optimal arms, but which cannot all be statistically
separated in $T$ rounds.
Assume without loss of generality that $q_1 \leq q_2 \leq \ldots \leq q_N \leq 1/2$.
For each $i$ define reward function $r_i$ by
\eq{
r_0(\boldsymbol{X}) &= \frac{1}{2} &
r_i(\boldsymbol{X}) &= \begin{cases}
\frac{1}{2} + \epsilon & \text{if } X_i = 1 \\
\frac{1}{2} & \text{otherwise}\,,
\end{cases}
}
where $1/4 \geq \epsilon > 0$ is some constant to be chosen later.
We abbreviate $R_{T,i}$ to be the expected simple regret incurred when interacting with the
environment determined by $\boldsymbol{q}$ and $r_i$. Let $\operatorname{P}_i$ be the corresponding measure
on all observations over all $T$ rounds and $\EE_i$ the expectation with respect to $\operatorname{P}_i$. By Lemma 2.6 by \citet{Tsy08} we have
\eq{
\Prz{\hat a^*_T = a^*} + \Pri{\hat a^*_T \neq a^*} \geq \exp\left(-\KL(\operatorname{P}_0, \operatorname{P}_i)\right)\,,
}
where $\KL(\Ps_0, \Ps_i)$ is the KL divergence between measures $\operatorname{P}_0$ and $\operatorname{P}_i$.
Let $T_i(T) = \sum_{t=1}^T \ind{a_t = do(X_i = 1)}$ be the total number of times the learner intervenes on variable $i$ by setting it to $1$.
Then for $i \leq m$ we have $q_i \leq 1/m$ and the KL divergence between $\Ps_0$ and $\Ps_i$ may be bounded using the telescoping property (chain rule) and
by bounding the local KL divergence by the $\chi$-squared distance as by \citet{Auer1995}. This leads to 
\eq{
\KL(\Ps_0, \Ps_i) 
&\leq 6\epsilon^2 \EE_0\left[\sum_{t=1}^T \ind{X_{t,i} = 1}\right] 
\leq 6\epsilon^2 \left(\EE_0 T_i(T) + q_i T\right) 
\leq 6\epsilon^2 \left(\EE_0 T_i(T) + \frac{T}{m}\right)\,.
}
Define set $A = \set{i \leq m : \EE_0 T_i(T) \leq 2T / m}$.
Then for $i \in A$ and choosing $\epsilon = \min\set{1/4, \sqrt{m/(18T)}}$ we have
\eq{
\KL(\Ps_0, \Ps_i) \leq \frac{18T\epsilon^2}{m} = 1\,. 
}
Now $\sum_{i=1}^m \EE_0 T_i(T) \leq T$, which implies that $|A| \geq m/2$.
Therefore
\eq{
\sum_{i \in A} \Pri{\hat a^*_T \neq a} 
\geq \sum_{i \in A} \exp\left(-\KL(\Ps_0, \Ps_i)\right) - 1
\geq \frac{|A|}{e} - 1 
\geq \frac{m}{2e} - 1\,.
}
Therefore there exists an $i \in A$ such that
$\Pri{\hat a^*_T \neq a^*} \geq \frac{\frac{m}{2e} - 1}{m}$. 
Therefore if $\epsilon < 1/4$ we have
\eq{
R_{T,i} \geq \frac{1}{2} \Pn{i}{\hat a^*_T \neq a^*} \epsilon \geq \frac{\frac{m}{2e} - 1}{2m} \sqrt{\frac{m}{18T}}\,.
}
Otherwise $m \geq 18T$ so $\sqrt{m/T} = \Omega(1)$ and
\eq{
R_{T,i} \geq \frac{1}{2} \Pn{i}{\hat a^*_T \neq a^*} \epsilon \geq \frac{1}{4} \frac{\frac{m}{2e} - 1}{2m} \in \Omega(1) 
}
as required.

\section{Proof of Theorem \ref{thm:general}}\label{sec:thm:general}

\begin{proof}
First note that $X_t, Y_t$ are sampled from $\operatorname{Q}$.
We define $Z_a(X_t) = Y_t R_a(X_t)\ind{R_a(X_t)\leq B_a}$ and abbreviate $Z_{at} = Z_a(X_t)$, $R_{at} = R_a(X_t)$ and $\Pn{a}{.} = \Pns{a}{.}$.
By definition we have $|Z_{at}| \leq B_a$ and 
\eq{
\Var_Q[Z_{at}] 
\leq \EE_Q[Z_{at}^2] 
\leq \EE_Q[R_{at}^2]
= \EEa[R_{at}]
= \EEa\left[\frac{\Pns{a}{\parents{Y}(X)}}{\Q{\parents{Y}(X)}}\right] 
\leq m(\eta)\,.
}
Checking the expectation we have
\eq{
\EE_Q[Z_{at}] 
= \EEa \left[Y \ind{R_{at} \leq B_a}\right] 
= \EEa Y - \EEa \left[Y\ind{R_{at} > B_a}\right] 
= \mu_a - \beta_a\,,
}
where 
\eq{
0 \leq \beta_a = \EEa[Y \ind{R_{at} > B_a}] \leq \Pns{a}{R_{at} > B_a}
}
is the negative bias. 
The bias may be bounded in terms of $m(\eta)$ via an application of Markov's inequality.
\eq{
\beta_a \leq \Pns{a}{R_{at} > B_a} \leq \frac{\EEa[R_{at}]}{B_a} \leq \frac{m(\eta)}{B_a}\,.
}
Let $\epsilon_a > 0$ be given by
\eq{
\epsilon_a = \sqrt{\frac{2m(\eta)}{T} \log\left(2T|\calA|\right)} + \frac{3B_a}{T} \log\left(2T|\calA|\right)\,.
}
Then by the union bound and Bernstein's inequality 
\eq{
\P{\text{exists } a \in \calA : \left|\hat \mu_a - \EE_Q[Z_{at}]\right| \geq \epsilon_a} 
\leq \sum_{a \in \calA} \P{\left|\hat \mu_a - \EE_Q[Z_{at}]\right| \geq \epsilon_a} \leq \frac{1}{T}\,.
}

Let $I = \hat{a}^*_T$ be the action selected by the algorithm, $a^* = \argmax_{a \in \calA} \mu_a$ be the true optimal action and recall that $\EE_Q[Z_{at}] = \mu_a - \beta_a$. Assuming the above event does not occur we have,

\eq{
\mu_I \geq \hat \mu_I - \epsilon_I  
\geq \hat \mu_{a^*} - \epsilon_I  
\geq \mu^* - \epsilon_{a^*} - \epsilon_I - \beta_{a^*}\,. 
}
By the definition of the truncation
we have
\eq{
\epsilon_a \leq \left(\sqrt{2} + 3\right)\sqrt{\frac{m(\eta)}{T} \log\left(2T|\calA|\right)}
}
and
\eq{
\beta_a \leq \sqrt{\frac{m(\eta)}{T} \log\left(2T|\calA|\right)}\,. 
}
Therefore for $C = \sqrt{2} + 4$ we have
\eq{
\P{\mu_I \geq \mu^* - C \sqrt{\frac{m(\eta)}{T} \log\left(2T|\calA|\right)}} \leq \frac{1}{T}\,.
}
Therefore
\eq{
\mu^* - \EE[\mu_I] \leq C \sqrt{\frac{m(\eta)}{T} \log\left(2T|\calA|\right)} + \frac{1}{T}
}
as required.
\end{proof}

\subsection{Relationship between $m(\eta)$ and $m(\boldsymbol{q})$}\label{sec:m-equivelence}

\begin{proposition} In the parallel bandit setting,
$m(\eta^*) \leq 2m(\boldsymbol{q})$.
\end{proposition} 

\begin{proof}

Recall that in the parallel bandit setting,

\eq{
\mathcal{A} = \set{do()} \cup \set{ do(X_i = j) \colon 1 \leq i \leq N \text{ and } j \in \set{0,1}}
}

Let:

\eq {
\eta_a = \ind{\P{X_i = j} < \frac{1}{m(\boldsymbol{q})}}\frac{1}{2m(\boldsymbol{q})} \text { for } a \in do(X_i = j)
}

Let $D =\sum_{a\in do(X_i=j)}\eta_a$. From the definition of $m(\boldsymbol{q})$, 
\eq {
\sum_{a\in do(X_i=j)} \ind{\P{X_i = j} < \frac{1}{m(\boldsymbol{q})}} \leq m(\boldsymbol{q}) \implies D \leq \frac{1}{2}
}
 
Let $\eta_a = \frac{1}{2} + (1-D)$ for $a = do()$ such that $\sum_{a \in \calA}\eta_a = 1$ 

Recall that,

\eq{
m(\eta) &
= \max_a \EEa\left[\frac{\Pn{a}{\parents{Y}(X)}}{\Q{\parents{Y}(X)}}\right]
}

We now show that our choice of $\eta$ ensures $\EEa\left[\frac{\Pn{a}{\parents{Y}(X)}}{\Q{\parents{Y}(X)}}\right] \leq 2m(\boldsymbol{q})$ for all actions $a$.

For the actions $a: \eta_a > 0$, ie $do()$ and $do(X_i = j):\P{X_i=j}<\frac{1}{m(\boldsymbol{q})}$,
\eq{
\EEa\left[\frac{\Pn{a}{X_1...X_N}}{\sum_{b}\eta_b\Pn{b}{X_1...X_N}}\right] \leq \EEa\left[\frac{\Pn{a}{X_1...X_N}}{\eta_a\Pn{a}{X_1...X_N}}\right] = \EEa\left[\frac{1}{\eta_a}\right] \leq 2m(\boldsymbol{q})
}

For the actions $a :\eta_a = 0$, ie $do(X_i=j):\P{X_i=j}\geq\frac{1}{m(\boldsymbol{q})}$,
\eq{
\EEa\left[\frac{\Pn{a}{X_1...X_N}}{\sum_{b}\eta_b\Pn{b}{X_1...X_N}}\right] \leq & \EEa\left[\frac{\ind{X_i=j}\prod_{k\neq i}\P{X_k}}{(1/2+D)\prod_k \P{X_k}}\right] \\=& \EEa\left[\frac{\ind{X_i=j}}{(1/2+D)\P{X_i = j}}\right]
\leq  \EEa\left[\frac{\ind{X_i=j}}{(1/2)(1/m(\boldsymbol{q}))}\right] \leq 2m(\boldsymbol{q})
}

Therefore $m(\eta*) \leq m(\eta) \leq 2m(\boldsymbol{q})$ as required.

\end{proof}
\

\fi

\end{document}